\documentclass[10pt,twocolumn,letterpaper]{article}

\usepackage{iccv}
\usepackage{times}
\usepackage{epsfig}
\usepackage{graphicx}
\usepackage{amsmath}
\usepackage{amssymb}

\usepackage{amsthm}
\usepackage{bm}
\usepackage{mathtools}
\usepackage{algorithm}
\usepackage{algpseudocode}
\usepackage{varwidth}
\usepackage{tabularx}
\usepackage{subcaption}

\usepackage{siunitx}
\newtheoremstyle{break}
  {\topsep}{\topsep}%
  {\itshape}{}%
  {\bfseries}{}%
  {\newline}{}%

\theoremstyle{plain}
\newtheorem{thm}{Theorem}[section] 

\theoremstyle{definition}

\theoremstyle{break}

\newcommand*\mat[1]{\mathbf{#1}}
\newcommand*\vv[1]{\mathbf{#1}}
\newcommand*\ten[1]{\bm{\mathcal{#1}}}

\newcommand{\shrink}{\mathcal{S}}
\newcommand{\vvec}{\mathrm{vec}}

\newcommand*\inner[2]{\langle #1, \, #2 \rangle}

\newcommand*\trp[1]{ #1^{\mathsf{T}} }

\newcommand*\rk[1]{\mathrm{rank}(#1)}

\newcommand{\st}{\mathrm{s.t}}

\newcommand{\RR}{\mathbb{R}}

\newcommand*\Fro[1]{ || #1 ||_{\mathrm{F}} }
\newcommand*\Nuc[1]{ || #1 ||_{*} }
\newcommand*\One[1]{ || #1 ||_{1} }
\newcommand*\Two[1]{ || #1 ||_{2} }
\newcommand*\Twosq[1]{ || #1 ||_{2}^2 }

\newcommand*\Frosq[1]{ || #1 ||_{\mathrm{F}}^2 }

\newcommand{\A}{\mat{A}}
\newcommand{\At}{\trp{\A}}
\newcommand{\B}{\mat{B}}
\newcommand{\Bt}{\trp{\B}}
\newcommand{\C}{\mat{C}}
\newcommand{\D}{\mat{D}}
\newcommand{\R}{\mat{R}}

\newcommand{\X}{\mat{X}}

\newcommand{\Z}{\mat{Z}}

\newcommand{\E}{\mat{E}}
\newcommand{\LL}{\mat{\Lambda}}
\newcommand{\II}{\mat{I}}

\newcommand{\U}{\mat{U}}
\newcommand{\V}{\mat{V}}

\newcommand{\Ri}{\mat{R}_i}

\newcommand{\Ei}{\mat{E}_i}

\newcommand{\XXi}{\mat{X}_i}

\newcommand{\Yi}{\mat{Y}_i}

\newcommand{\LLi}{\mat{\Lambda}_i}
\newcommand{\Ki}{\mat{K}_i}

\newcommand{\tX}{\ten{X}}

\newcommand{\tE}{\ten{E}}
\newcommand{\tL}{\ten{L}}

\newcommand{\tY}{\ten{Y}}

\newcommand{\tK}{\ten{K}}
\newcommand{\tR}{\ten{R}}

\newcommand\Tstrut{\rule{0pt}{2.6ex}}         
\newcommand\Bstrut{\rule[-0.9ex]{0pt}{0pt}}   

\newcommand*\mypar[1]{\vspace{0.5em} \noindent \textbf{#1} \hspace{0.5em}}

\usepackage[pagebackref=true,breaklinks=true,letterpaper=true,colorlinks,bookmarks=false]{hyperref}

\iccvfinalcopy 

\def\iccvPaperID{1329} 

\ificcvfinal\fi
\begin{document}

\title{Robust Kronecker-Decomposable Component Analysis for Low-Rank Modeling}

\author{Mehdi Bahri\textsuperscript{1} \quad \quad Yannis Panagakis\textsuperscript{1,2} \quad \quad Stefanos Zafeiriou\textsuperscript{1,3}\\
\textsuperscript{1}Imperial College London, UK \quad \textsuperscript{2}Middlesex University, London, UK \quad \textsuperscript{3}University of Oulu, Finland\\
{\tt\small mehdi.b.tn@gmail.com, \{i.panagakis, s.zafeiriou\}@imperial.ac.uk}
}

\maketitle

\begin{abstract}
 Dictionary learning and component analysis are part of one of the most well-studied and active research fields, at the intersection of signal and image processing, computer vision, and statistical machine learning. In dictionary learning, the current methods of choice are arguably K-SVD and its variants, which learn a dictionary (i.e\onedot, a decomposition) for sparse coding via Singular Value Decomposition. In robust component analysis, leading methods derive from Principal Component Pursuit (PCP), which recovers a low-rank matrix from sparse corruptions of unknown magnitude and support.  However, K-SVD is sensitive to the presence of noise and outliers in the training set. Additionally, PCP does not provide a dictionary that respects the structure of the data (e.g\onedot, images), and requires expensive SVD computations when solved by convex relaxation. In this paper, we introduce a new robust decomposition of images by combining ideas from sparse dictionary learning and PCP. We propose a novel Kronecker-decomposable component analysis which is robust to gross corruption, can be used for low-rank modeling, and leverages separability to solve significantly smaller problems. We design an efficient learning algorithm by drawing links with a restricted form of tensor factorization. The effectiveness of the proposed approach is demonstrated on real-world applications, namely  background subtraction and image denoising, by performing  a thorough comparison with the current state of the art.
\end{abstract}

\section{Introduction}

Sparse dictionary learning \cite{Rubinstein2010, Wright2010, Olshausen1997} and Robust Principal Component Analysis (RPCA) \cite{Candes2011a} are two popular methods of learning representations that assume different structure and answer different needs, but both are special cases of structured matrix factorization. Assuming a set of $N$ data samples $\mathbf{x}_1,\ldots,\mathbf{x}_n$ represented as the columns of a matrix $\mathbf{X}$, we seek a decomposition of $\X$ into meaningful components of a given structure, in the generic form:
\begin{equation}\label{eq:structured_mf_generic}
\min_\Z l(\X, \Z) + g(\Z)
\end{equation}
Where $l(\cdot)$ is a loss function, generally the $\ell_2$ error, and $g(\cdot)$ a possibly non-smooth regularizer that encourages the desired structure, often in the form of specific sparsity requirements.

When $\Z$ is taken to be factorized in the form $\D\R$, we obtain a range of different models depending on the choice of the regularization. Imposing sparsity on the dictionary $\D$ leads to sparse PCA, while sparsity only in the code $\R$ yields sparse dictionary learning. The current methods of choice for sparse dictionary learning are K-SVD \cite{Aharon2006} and its variants \cite{Mairal2008}, which seek an overcomplete dictionary $\mathbf{D}$ and a sparse representation $\mathbf{R} =[\mathbf{r}_1,\ldots,\mathbf{r}_n]$ by minimizing the following constrained objective:
\begin{equation}\label{eq:matching_pursuit_problem}
    \min_{\mathbf{R},\mathbf{D}} ||\mathbf{X}-\mathbf{D}\mathbf{R}||_\mathrm{F}^2,\,\, \mbox{s.t.}\,\, ||\mathbf{r}_i||_0 \leq T_0,
\end{equation}
where $||.||_\mathrm{F}$ is the  Frobenius norm and $||.||_0$ is the $\ell_0$ pseudo-norm, counting the number of non-zero elements. Problem (\ref{eq:matching_pursuit_problem}) is solved in an iterative manner that alternates between sparse coding of the data samples on the current dictionary, and a process of updating the dictionary atoms to better fit the data using the Singular Value Decomposition (SVD). When used to reconstruct images, K-SVD is trained on overlapping image patches to allow for overcompleteness \cite{Elad2006}. Problems of this form suffer from a high computational burden that limits their applicability to small patches that only capture local information, and prevent them from scaling to larger images. 

The recent developments in robust component analysis are attributed to the advancements in compressive sensing \cite{Candes2006, Donoho2006, Candes2005DecodingProgramming}.  In this area, the emblematic model is the Robust Principal Component Analysis, proposed in \cite{Candes2011a}. RPCA assumes that the observation matrix $\mathbf{X}$ is the sum of a low-rank component $\mathbf{A}$, and of a sparse matrix $\mathbf{E}$ that collects the gross errors, or outliers. Finding the minimal rank solution for $\A$ and the sparsest solution for $\E$ is combinatorial and NP-hard. Therefore, RPCA is obtained by solving the Principal Component Pursuit (PCP) in which the discrete rank and  $\ell_0$-norm functions are approximated by their  convex envelopes as follows:
\begin{equation}\label{eq:cvx_rpca}
    \min_{\mathbf{A},\mathbf{E}} ||\mathbf{A}||_* + ||\mathbf{E}||_1,\,\, \mbox{s.t.}\,\, \mathbf{X} = \mathbf{A} + \mathbf{E}
\end{equation}
where $||.||_*$ is the nuclear norm and $||.||_1$ is the standard $\ell_1$ norm. Like sparse dictionary learning, Robust PCA is a special case of (\ref{eq:structured_mf_generic}). The nuclear-norm relaxation is convex but the cost of performing SVDs on the full-size matrix $\A$ is high. Factorization-based formulations exploit the fact that a rank-$r$ matrix $\A \in \RR^{m \times n}$ can be decomposed in $\A = \U \trp{\V}$, $\U \in \RR^{m \times r}, \V \in \RR^{n \times r}$, to impose low-rankness. These formulations are non-convex and algorithms may get stuck in local optima or saddle points. We refer to \cite{DBLP:journals/corr/LiWLAHLZ16} for recent developments on the convergence of non-convex matrix factorizations, \cite{2016arXiv160306610W, 2016arXiv160603168P, 2017arXiv170207945Z, 2017arXiv170206525Z} for low-rank problems, and \cite{DBLP:conf/nips/GeLM16} specifically for matrix completion.

\subsection{Separable dictionaries}


Analytical separable dictionaries have been proposed as generalizations of 1-dimensional transforms in the form, \eg, of tensor-product wavelets, and have since fallen out of flavour for models that drop orthogonality in favour of overcompleteness to achieve geometric invariance \cite{Rubinstein2010}. However, recent work in the compressed sensing literature has shown a regain of interest for separable dictionaries for very high-dimensional signals where scalability is a hard-requirement. Notably, high-resolution spatial angular representations of diffusion-MRI signals, such as in HARDI (\textit{High Angular Resolution Diffusion Imaging}), represent each voxel by a 3D signal, yielding prohibitive sampling times. Traditional compressed sensing techniques can only handle signal sparsity bounded below by the number of voxels (1 atom per voxel) and still cannot meet the needs of medical imaging, but Kronecker extensions \cite{Schwab2016, 2016arXiv161205846S} of Orthogonal Matching Pursuit (OMP) \cite{342465, doi:10.1117/12.173207}, Dual ADMM \cite{Boyd2010}, and FISTA \cite{Beck2009} allow sparser signals.

The recent Separable Dictionary Learning (SeDiL) \cite{Hawe2013} considers a dictionary that factorizes into the Kronecker product of two smaller dictionaries $\A$ and $\B$, and matrix observations $\mathcal{X} = (\XXi)_i$. The observations have sparse representations $\mathcal{R} = (\Ri)_i$ in the bases $\A, \B$, and the learning problem is recast as:
\begin{equation}\label{eq:sedil}
\min_{\A, \B, \tR} \frac{1}{2} \sum_i \Frosq{\XXi - \A \Ri \Bt} + \lambda g(\tR) + \kappa r(\A) + \kappa r(\B)
\end{equation}
Where the regularizers $g$ and $r$ promote, respectively, sparsity in the representations, and low mutual-coherence of the dictionary $\D = \B \otimes \A$. Here, $\D$ is constrained to have orthogonal columns, \ie, the pair $\A, \B$ shall lie on the product manifold of two product of sphere manifolds. A different approach is taken in \cite{Hsieh2014}. A separable 2D dictionary is learnt in a two-step strategy similar to that of K-SVD. Each matrix observation $\XXi$ is represented as $\A \Ri \Bt$. In the first step, the sparse representations $\Ri$ are found by 2D OMP. In the second step, a CP decomposition is performed on a tensor of residuals via Regularized Alternating Least Squares to solve $\min_{\A, \B, \tR} \Fro{\tX - \tR \times_1 \A \times_2 \B}$\footnote{\cf notations for the product $\times_n$}.

\subsection{Contributions}

In this paper, we propose a novel method for separable dictionary learning based on a robust tensor factorization that learns simultaneously the dictionary and the sparse representations. We extend the previous approaches by introducing an additional level of structure to make the model robust to outliers. We do not seek overcompleteness, but rather  promote sparsity in the dictionary to learn a low-rank representation of the input tensor. In this regard, our method combines ideas from both Sparse Dictionary Learning and Robust PCA. We propose a non-convex parallelizable ADMM algorithm and provide experimental evidence of its effectiveness. Finally, we compare the performance of our method against several tensor and matrix factorization algorithms on computer vision benchmarks, and show our model systematically matches or outperforms the state of the art. We make the code available online\footnote{https://github.com/mbahri/KDRSDL} along with supplementary material.

\mypar{Notations}
Throughout the paper, matrices (vectors) are denoted by uppercase (lowercase) boldface letters \eg, $\mathbf{X}$, ($\mathbf{x}$). 
$\mathbf{I}$ denotes the identity matrix of compatible dimensions. The $i^{th}$ column of $\mathbf{X}$ is denoted as $\mathbf{x}_{i}$. Tensors are considered as the multidimensional equivalent of matrices (second-order tensors), and vectors (first-order tensors),
and denoted by bold calligraphic letters, \eg,  $\bm{\mathcal{X}}$. The
\textit{order} of a tensor is the number of indices needed to
address its elements. Consequently, each element of an $M^{th}$-order
tensor $\bm{\mathcal{X}}$ is addressed by $M$ indices, \ie,
$(\bm{\mathcal{X}})_{i_{1}, i_{2}, \ldots, i_{M}} \doteq x_{i_{1}, i_{2}, \ldots, i_{M}}$.

The sets of real and integer numbers are denoted by $\mathbb{R}$ and $\mathbb{Z}$, respectively.    An $M^{th}$-order  real-valued tensor $\bm{\mathcal{X}}$ is  defined over the
tensor space $\mathbb{R}^{I_{1} \times I_{2} \times \cdots \times
I_{M}}$, where $I_{m} \in \mathbb{Z}$ for $m=1,2,\ldots,M$. 

The mode-$n$  product of a tensor  $\bm{\mathcal{X}} \in
\mathbb{R}^{I_{1}\times I_{2}\times \ldots \times I_{M}}$ with a matrix $\mathbf{U} \in \mathbb{R}^{J \times I_{n}}$, denoted by
$\bm{\mathcal{X}} \times_{n} \mathbf{U}$, is defined element-wise as
\begin{equation}\label{E:Tensor_Mode_n}
(\bm{\mathcal{X}} \times_{m} \mathbf{x})_{i_1, \ldots, i_{n-1}, i_{n+1},
\ldots, i_{M}} = \sum_{i_n=1}^{I_n} x_{i_1, i_2, \ldots, i_{M}} x_{i_n}.
\end{equation}
The \textit{Kronecker } product of matrices $\mathbf{A} \in \mathbb{R}^{I \times K}$,
and $\mathbf{B} \in \mathbb{R}^{L \times M}$, is
denoted by $\mathbf{A} \otimes \mathbf{B}$, and yields a matrix of dimensions $I\cdot L\times K \cdot M$.

Finally, we define the tensor \textit{Tucker rank} as the vector of the ranks of its mode-$n$ unfoldings (\ie, its mode-$n$ ranks), and the tensor \textit{multi-rank} \cite{Zhang2014,Lu_2016_CVPR} as the vector of the ranks of its frontal slices. More details about tensors, such as the definitions of tensor slices and mode-$n$ unfoldings, can be found in  \cite{Kolda2009} for example.


\section{Model and updates}


We first formulate our structured factorization, and describe a tailored optimization procedure.

\subsection{Robustness and structure}

Consider the following Sparse Dictionary Learning problem with Frobenius-norm regularization on the dictionary $\D$, where we decompose $N$ observations $\vv{x}_i \in \RR^{mn}$ on $\D \in \RR^{mn \times r_1 r_2}$ with representations $\vv{r}_i \in \RR^{r_1 r_2}$:
\begin{align}\label{eq:sparse_dictionary_learning}
    \min_{\D, \mat{R}}{ \sum_i \Twosq{\vv{x}_i - \D \vv{r}_i } } + \lambda \sum_i \One{\vv{r}_i} + \Fro{\D}
\end{align}
We assume a Kronecker-decomposable dictionary $\D = \B \otimes \A$ with $\A \in \RR^{m \times r_1}, \B \in \RR^{n \times r_2}$. To model the presence of outliers, we introduce a set of vectors $\vv{e}_i \in \RR^{mn}$ and, with $d = r_1 r_2 + mn$, define the block vectors and matrices:

\begin{equation}
\vv{y}_i = \begin{bmatrix}\vv{r}_i\\ \vv{e}_i\end{bmatrix}  \in \RR^{d} \quad \C = \begin{bmatrix} \B \otimes \A & \II \end{bmatrix} \in \RR^{mn \times d}
\end{equation}

We obtain a two-level structured dictionary $\mat{C}$ and the associated sparse encodings $\vv{y}_i$. Breaking-down the variables to reduce dimensionality and discarding the constant $\Fro{\II}$:

\begin{equation}
\begin{split}\label{eq:two_level_structured_sparse_dictionary_learning_explicit}
\min_{\A, \B, \mat{R}, \mat{E}} \sum_i \Twosq{\vv{x}_i - (\B \otimes \A) \vv{r}_i - \vv{e}_i } \\ + \lambda \sum_i \One{\vv{r}_i} + \lambda \sum_i \One{\vv{e}_i} + \Fro{\B \otimes \A}
\end{split}
\end{equation}
Suppose now that the observations $\vv{x}_i$ were obtained by vectorizing two-dimensional data such as images, \ie, $\vv{x}_i = \vvec({\XXi}), \XXi \in \RR^{m \times n}$. We find preferable to keep the observations in matrix form as this preserves the spatial structure of the images, and - as we will see - allows us solve matrix equations efficiently instead of high-dimensional linear systems. Without loss of generality, we choose $r_1 = r_2 = r$ and $\vv{r}_i = \vvec(\Ri), \Ri \in \RR^{r \times r}$, and recast the problem as:
\begin{equation}\label{eq:matrix_2d_rpca}
\begin{split}
\min_{\A, \B, \ten{R}, \ten{E}} \sum_i \Frosq{\XXi - \A \Ri \Bt - \Ei }  \\ 
    + \lambda \sum_i \One{\Ri} + \lambda \sum_i \One{\Ei} + \Fro{\B \otimes \A}
\end{split}
\end{equation}
Equivalently, enforcing the equality constraints and writing the problem explicitly as a structured tensor factorization:
\begin{align}\label{eq:constrained_pb}
    \begin{matrix*}[l]
    \min_{\A, \B, \ten{R}, \ten{E}} & \lambda \One{\ten{R}} + \lambda \One{\tE} + \Fro{\B \otimes \A}\\
    \st & \tX = \ten{R} \times_1 \A \times_2 \B + \tE
    \end{matrix*}
\end{align}
Where the matrices $\XXi, \Ri$, and $\Ei$ are concatenated as the frontal slices of 3-way tensors. Figure \ref{fig:decomposition} illustrates the decomposition.

\begin{figure}
\centering
\includegraphics[width=\columnwidth]{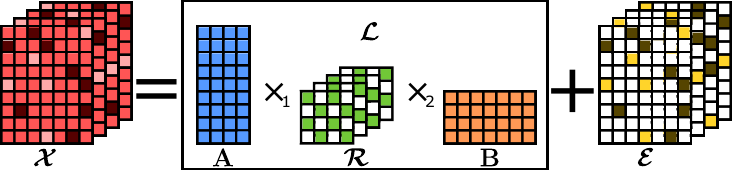}
\caption{Illustration of the decomposition.}
\label{fig:decomposition}
\end{figure}

We impose $r \leq \min(m, n)$ as a natural upper bound on the rank of the frontal slices of $\tL = \ten{R} \times_1 \A \times_2 \B$, and on its mode-$1$ and mode-$2$ ranks.


\subsection{An efficient algorithm}
Problem (\ref{eq:constrained_pb}) is not jointly convex, but is convex in each component individually. We resort to an alternating-direction method and propose a non-convex ADMM procedure that operates on the frontal slices.


Minimizing $\Fro{\B \otimes \A}$ presents a challenge: the product is high-dimensional, the two bases are coupled, and the loss is non-smooth. Using the identity $||\A \otimes \B||_p = ||\A||_p ||\B||_p$ where $||.||_p$ denotes the Schatten-$p$ norm (this follows from the compatibility of the Kronecker product with the singular value decomposition), and remarking that $\Fro{\B} \Fro{\A} \leq \frac{\Frosq{\A} + \Frosq{\B}}{2}$, we minimize a simpler upper bound ($\Nuc{\A\Bt}$ \cite{Recht2007}). The resulting sub-problems are smaller, and therefore more scalable. In order to obtain exact proximal steps for the encodings $\Ri$, we introduce a split variable $\Ki$ such that $\forall i, \; \Ki = \Ri$. Thus, we solve:
\begin{align}\label{eq:constrained_pb_rpca_upper_bound}
    \begin{matrix*}[l]
    \displaystyle \min_{\A, \B, \ten{R}, \tK \ten{E}} & \lambda \One{\ten{R}} + \lambda \One{\tE} + \frac{1}{2}(\Frosq{\A} + \Frosq{\B})\\
    \st & \tX = \ten{K} \times_1 \A \times_2 \B + \tE \\
    \st & \tR = \tK
    \end{matrix*}
\end{align}

Introducing the tensors of Lagrange multipliers $\LL$ and $\tY$, such that the $i^{th}$ frontal slice corresponds to the $i^{th}$ constraint, and the dual step sizes $\mu$ and $\mu_{\tK}$, we formulate the Augmented Lagrangian of problem (\ref{eq:constrained_pb_rpca_upper_bound}):
\begin{equation}\label{eq:augmented_lagrangian_pb_split}
\begin{split}
    \min \lambda \sum_i \One{\Ri} + \lambda \sum_i \One{\Ei} + \frac{1}{2}(\Frosq{\A} + \Frosq{\B}) + \\
    \sum_i \inner{\LLi}{\XXi - \A \Ki \Bt - \Ei} + 
    \sum_i \inner{\Yi}{\Ri - \Ki} + \\ \frac{\mu}{2} \sum_i \Frosq{\XXi - \A \Ki \Bt - \Ei } + \frac{\mu_{\tK}}{2} \sum_i \Frosq{\Ri - \Ki}
\end{split}
\end{equation}
We can now derive the ADMM updates. Each $\Ei$ is given by shrinkage after rescaling:
\begin{equation}
\Ei = \shrink_{\lambda/\mu} (\XXi - \A \Ki \Bt + \frac{1}{\mu} \LLi)
\end{equation}
A similar rule is immediate to derive for $\Ri$, and solving for $\A$ and $\B$ is straightforward with some matrix algebra. We therefore focus on the computation of the split variable $\Ki$. Differentiating, we find $\Ki$ satisfies:
\begin{align}\label{eq_updt_K_opti}
    & \mu_{\tK} \Ki + \mu \At\A \Ki \Bt\B \\
    & = \At ( \LLi + \mu ( \XXi - \Ei ) ) \B + \mu_{\tK} \Ri + \Yi \nonumber
\end{align}
The key for an efficient algorithm is here to recognize equation (\ref{eq_updt_K_opti}) is a \textit{Stein} equation, and can be solved in cubical time and quadratic space in $r$ by solvers for discrete-time Sylvester equations - such as the Hessenberg-Schur method \cite{Golub1979} - instead of the naive $O(r^6)$ time, $O(r^4)$ space solution of vectorizing the equation in an $r^2$ linear system.

In practice, we solve a slightly different problem with $\alpha \sum_i \One{\Ri} + \lambda \sum_i \One{\Ei}$. This introduces an additional degree of freedom and corresponds to rescaling the coefficients of $\Ri$. We found the modified problem to be numerically more stable and to allow for tuning of the relative importance of $\One{\tR}$ and $\One{\tE}$. We obtain Algorithm \ref{alg:kdrsdl}.
\begin{algorithm*}[ht]
\footnotesize
\begin{algorithmic}[1]
\Procedure{KDRSDL}{$\tX,r, \lambda, \alpha$}

\State $\A^0, \B^0, \tE^0, \ten{R}^0, \tK^0 \gets$ \Call{Initialize}{$\tX$}

\While{not converged}

\State $\tE^{t+1} \gets \shrink_{\lambda / \mu^t} ( \tX - \tK^{t} \times_1 \mat{A}^{t} \times_2 \mat{B}^{t} + \frac{1}{\mu^t} \LL^t)$

\State $\tilde{\tX}^{t+1} \gets \tX - \tE^{t+1}$

\State $\mat{A}^{t+1}\gets (\sum_i (\mu^t \tilde{\X}_i^{t+1} + \LLi^{t}) \mat{B}^{t} \trp{(\Ki^t)}) \; / \; (\II + \mu^t \sum_i \Ki^t \trp{(\mat{B}^{t})} \mat{B}^{t} \trp{(\Ki^t)})$

\State $\mat{B}^{t+1}\gets (\sum_i \trp{(\mu^t \tilde{\X}_i^{t+1} + \LLi^t)} \mat{A}^{t+1} \Ki^t) \; / \; (\II + \mu^t \sum_i \trp{(\Ki^t)} \trp{(\mat{A}^{t+1})} \mat{A}^{t+1} \Ki^t)$

\ForAll{i}
    \State
        \begin{varwidth}[t]{\linewidth}
            $\mat{K}_i^{t+1} \gets$ \textsc{Stein}($-\frac{\mu^t}{\mu_{\tK}^t} \trp{(\mat{A}^{t+1})}\mat{A}^{t+1}$,
                \; $\trp{(\mat{B}^{t+1})} \mat{B}^{t+1}$, $\frac{1}{\mu_{\tK}^t} \left[ \trp{(\mat{A}^{t+1})} ( \mat{\Lambda}_i^{t} + \mu^t \tilde{\mat{X}}_i^{t+1} ) \mat{B}^{t+1} + \mat{Y}_i^t \right] + \Ri^t $)
        \end{varwidth}
    \State $\mat{R}_i^{t+1} \gets \shrink_{\alpha / \mu_{\tK}^t} (\mat{K}_i^{t+1} - \frac{1}{\mu_{\tK}^t} \mat{Y}_i^t)$
\EndFor

\State $\LL^{t+1} \gets \LL^{t} + \mu^t (\tilde{\tX}^{t+1} - \tK^{t+1} \times_1 \mat{A}^{t+1} \times_2 \mat{B}^{t+1})$
\State $\tY^{t+1} \gets \tY^{t} + \mu_{\tK}^t (\tR^{t+1} - \tK^{t+1})$
\State $\mu^{t+1} \gets \min(\mu^*,  \rho \mu^t)$
\State $\mu_{\tK}^{t+1} \gets \min(\mu_{\tK}^*, \rho \mu_{\tK}^t)$

\EndWhile

\State \textbf{return} $\A, \B, \ten{R}, \tE$ 
\EndProcedure
\end{algorithmic}
\caption{Kronecker-Decomposable Robust Sparse Dictionary Learning.}
\label{alg:kdrsdl}
\end{algorithm*}

%
%
%
%

\section{Discussion}

We show experimentally that our algorithm successfully recovers the components of the decomposition (\ref{eq:constrained_pb}), we then prove the formulation encourages two different notions of low-rankness on tensors, and finally discuss the issues of non-convexity and convergence.

\subsection{Validation on synthetic data}

We generated synthetic data following the model's assumptions by first sampling two random bases $\A$ and $\B$ of known ranks $r_{\A}$ and $r_{\B}$, $N$ Gaussian slices for the core $\tR$, and forming the ground truth $\tL = \tR \times_1 \A \times_2 \B$. We modeled additive random sparse Laplacian noise with a tensor $\tE$ whose entries are $0$ with probability $p$, and $1$ or $-1$ with equal probability otherwise. We generated data for $p = 70\%$ and $p = 40\%$, leading to a noise density of, respectively, $30\%$ and $60\%$. We measured the reconstruction error on $\tL$, $\tE$, and the density of $\tE$ for varying values of $\lambda$, and $\alpha = \num{1e-2}$. Our model achieved near-exact recovery of both $\tL$ and $\tE$, and exact recovery of the density of $\tE$, for suitable values of $\lambda$. Experimental evidence is presented in Figure \ref{fig:synthdata_60} for the $60\%$ noise case.

\begin{figure}
\centering
\includegraphics[width=\columnwidth]{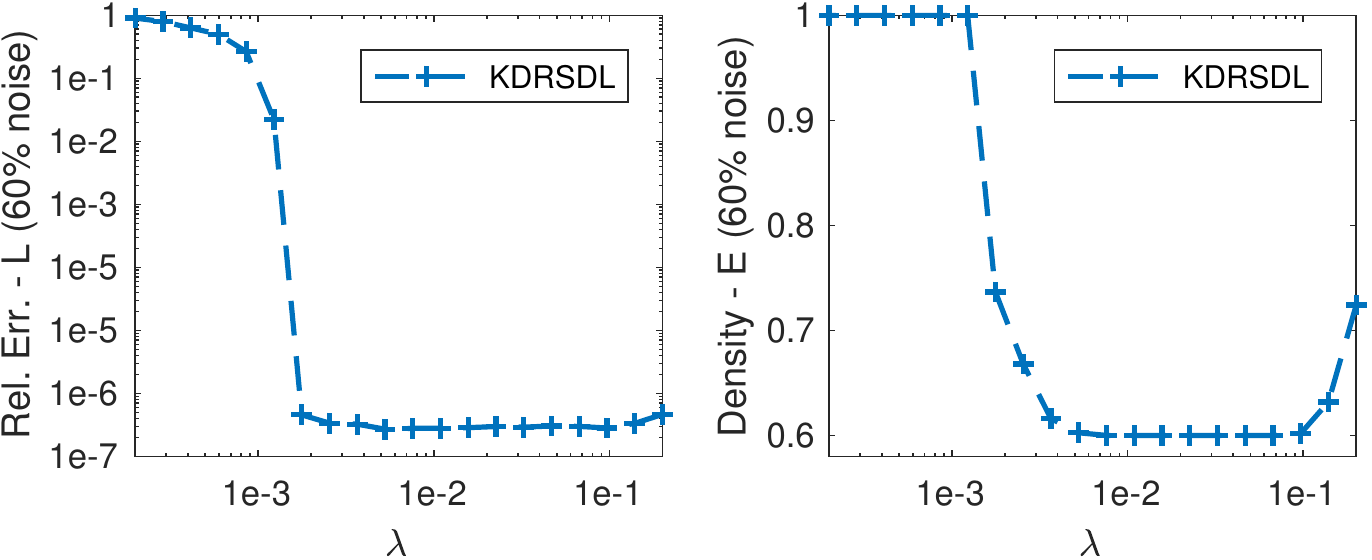}
\caption{Recovery performance with 60\% corruption. Relative $\ell_2$ error and density.}
\label{fig:synthdata_60}
\end{figure}

The algorithm appears robust to small changes in $\lambda$, which suggests not only one value can lead to optimal results, and that a simple criterion that provides consistently good reconstruction may be derived, as in Robust PCA \cite{Candes2011a}. In the $30\%$ noise case, we did not observe an increase in the density of $\tE$ as $\lambda$ increases, and the $\ell_2$ error on both $\tE$ and $\tL$ was of the order of \num{1e-7} (\cf supplementary material).

\subsection{Low-rank solutions}
Seeing the model from the perspective of Robust PCA (\ref{eq:cvx_rpca}), which seeks a low-rank representation $\A$ of the dataset $\X$, we minimize the rank of the low-rank tensor $\tL$. More precisely, we show in theorem \ref{thm:rank} that we simultaneously penalize the Tucker rank and the multi-rank of $\tL$.

\begin{thm}\label{thm:rank}
Algorithm \ref{alg:kdrsdl} encourages low mode-$1$ and mode-$2$ rank, and thus, low-rankness in each frontal slice of $\tL$, for suitable choices of the parameters $\lambda$ and $\alpha$.
\end{thm}
\begin{proof}
From the equivalence of norms in finite-dimensions, $\exists k \in \RR_+^*, \Nuc{\A \otimes \B} \leq k \Fro{\A \otimes \B}$. We minimise $\lambda \One{\tE} + \alpha \One{\tR} + \Fro{\A \otimes \B}$. By choosing $\alpha = \frac{\alpha'}{k}, \lambda = \frac{\lambda'}{k}$, we penalize $\rk{\A \otimes \B} = \rk{\A}\rk{\B}$. Given that the rank is a non-negative integer, $\rk{\A}$ or $\rk{\B}$ decreases necessarily. Therefore, we minimize the mode-$1$ and mode-$2$ ranks of $\tL = \tR \times_1 \A \times_2 \B$. Additionally, $\forall i, \rk{\A \Ri \Bt} \leq \min(\rk{\A}, \rk{\B}, \rk{\Ri})$.

Exhibiting a valid $k$ may help in the choice of parameters. We show $\forall \A \in \RR^{m \times n}, \; \Nuc{\A} \leq \sqrt{\min(m, n)}\Fro{\A}$: We know $\forall \vv{x} \in \RR^n, \One{\vv{x}} \leq \sqrt{n} \Two{\vv{x}}$. By definition, the Schatten-$p$ norm of $\A$ is the $\ell_p$ norm of its singular values. Recalling the nuclear norm and the Frobenius norm are the Schatten-$1$ and Schatten-$2$ norms, $\A$ has $\min(m, n)$ singular values, hence the result.
\end{proof}
In practice, we find that on synthetic data designed to test the model, we effectively recover the ranks of $\A$ and $\B$ regardless of the choice of $r$, as seen in Figure \ref{fig:sample_ranks}.
\begin{figure}
\centering
\includegraphics[width=\columnwidth]{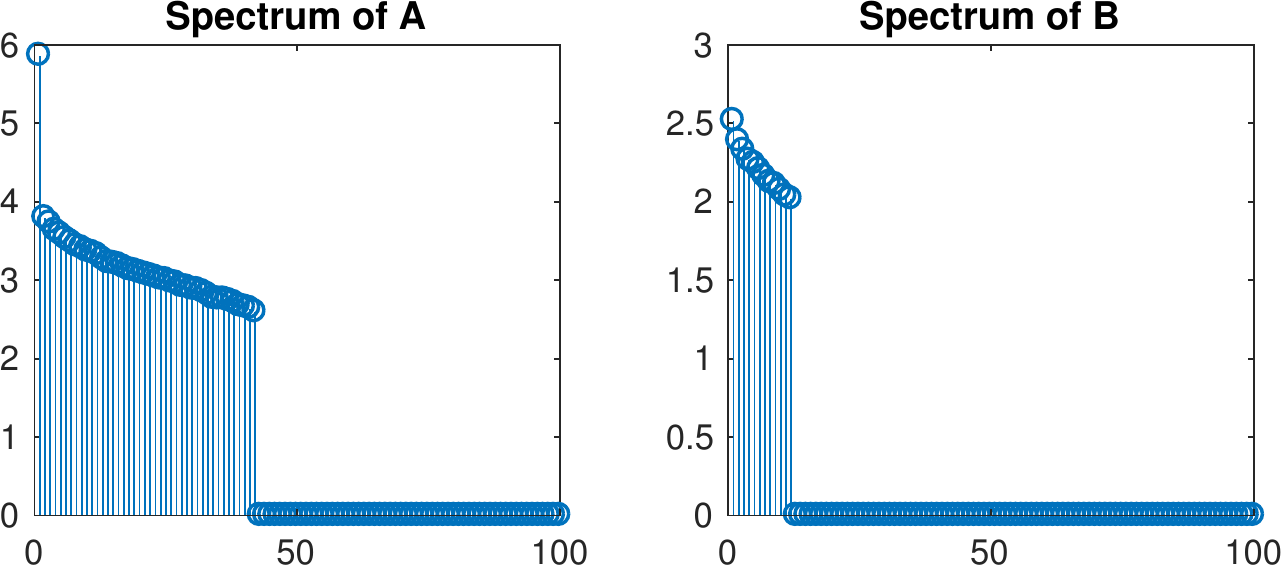}
\caption{Sample spectrums of $\A$ and $\B$, the ground truth is 42 for $\A$ and 12 for $\B$, and is attained. $r = 100$.}
\label{fig:sample_ranks}
\end{figure}

\subsection{Convergence and complexity}

\textit{Convergence and initialization:} The problem is non-convex, therefore global convergence is a priori not guaranteed. Recent work \cite{7178689, 2015arXiv151106324W} studies the convergence of ADMM for non-convex and possibly non-smooth objective functions with linear constraints. Here, the constraints are not linear. In \cite{icml2014c2_haeffele14, DBLP:journals/corr/HaeffeleV15} the authors derive conditions for global optimality in specific non-convex matrix and tensor factorizations that may be extended to other formulations, including our model. However, the theoretical study of the global convergence of our algorithm is out of the scope of this paper and is left for future work. Instead, we provide experimental results and discuss the strategy implemented.

We propose a simple initialization scheme in the wake of \cite{Lin2013}. We initialize the bases $\A$ and $\B$ and the core $\tR$ by performing SVD on each observation $\XXi = \mat{U}_i \mat{S}_i \trp{\mat{V}_i}$. We set $\Ri = \mat{S}_i$, $\A = \frac{1}{N} \sum_i \mat{U}_i$ and $\B = \frac{1}{N} \sum_i \mat{V}_i$. To initialize the dual-variables for the constraint $\XXi - \A \Ri \Bt - \Ei = \mat{0}$, we take $\mu^0 = \frac{\eta N}{\sum_i \Fro{\XXi}}$ where $\eta$ is a scaling coefficient, chosen in practice to be $\eta = 1.25$ as in \cite{Lin2013}. Similarly, we chose $\mu_{\ten{K}}^0 = \frac{\eta N}{\sum_i \Fro{\Ri}}$. These correspond to averaging the initial values for each individual slice and its corresponding constraint. Our convergence criterion corresponds to primal-feasibility of problem (\ref{eq:constrained_pb_rpca_upper_bound}), and is given by $\max(\mathrm{err}_{rec}, \mathrm{err}_{split}) \leq \epsilon$ where $\mathrm{err}_{rec} = \max_i \frac{\Frosq{\XXi - \A \Ri \Bt - \Ei}}{\Frosq{\XXi}}$ and $\mathrm{err}_{split} = \max_i \frac{\Frosq{\Ri - \Ki}}{\Frosq{\Ri}}$. Empirically, we obtained systematic convergence to a good solution, and a linear convergence rate, as shown in Figure \ref{fig:sample_convergence}. The parameter $\rho > 1$ can be tuned and affects the speed of convergence. High values may lead to the algorithm diverging, while values closer to 1 will lead to an increased number of iterations. We used $\rho = 1.2$ and set upper bounds $\mu^* = \mu^0 \times 10^7 $ and $\mu_{\tK}^* = \mu_{\tK}^0 \times 10^7 $ as usual with ADMM.

\begin{figure}
\centering
\includegraphics[width=\columnwidth]{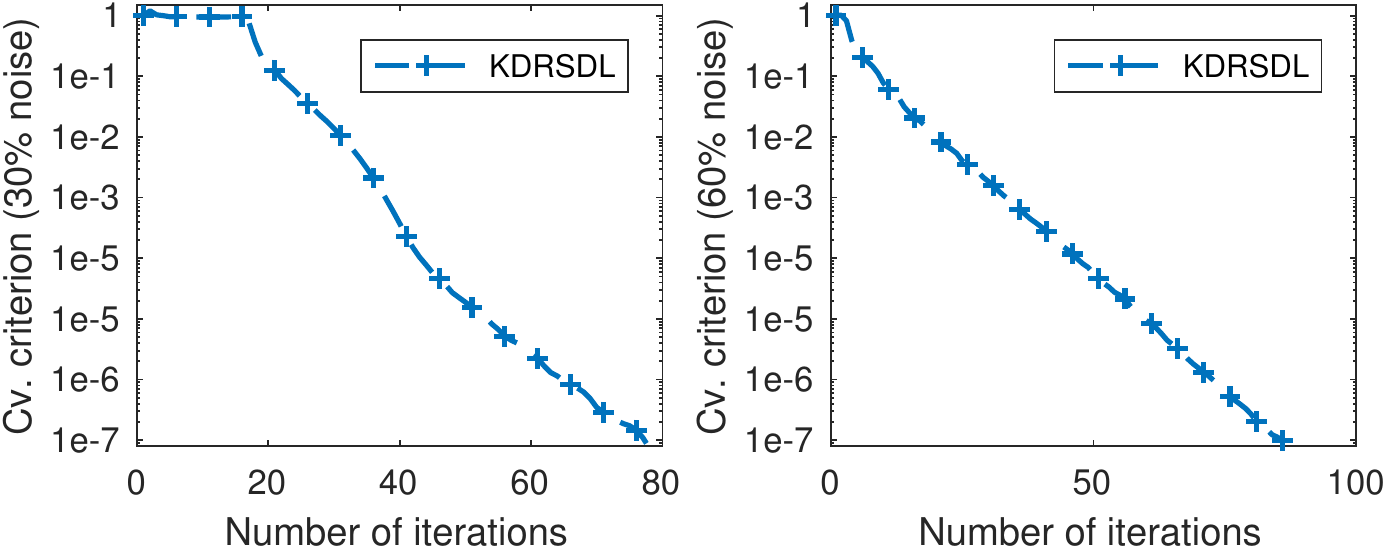}
\caption{Convergence on synthetic data with 30\% and 60\% corruption.}
\label{fig:sample_convergence}
\end{figure}

\textit{Computational complexity:} The time and space complexity per iteration of Algorithm \ref{alg:kdrsdl} are $O(N(mnr + (m + n)r + mn + \min(m, n)r^2 + r^3 + r^2))$ and $O( N(mn + r^2) + (m + n)r + r^2  )$. Since $r \leq \min(m, n)$, the terms in $r$ are asymptotically negligible, but in practice it is useful to know how the computational requirements scale with the size of the dictionary. Similarly, the initialization procedure has cost $O(N(m n \min (m, n) + (\min(m ,n))^3 + mn) + mn)$ in time and needs quadratic space per slice, assuming a standard algorithm is used for the SVD\cite{Chan1982}. Several key steps in the algorithm, such as summations of independent terms, are trivially distributed in a \textit{MapReduce} \cite{Dean2004} way. Proximal operators are separable in nature and are therefore parallelizable. Consequently, highly parallel and distributed implementations are possible, and computational complexity can be further reduced by adaptively adopting sparse linear algebra structures and algorithms.

\section{Experimental evaluation\protect}
We compared the performance of our model against a range of state-of-the-art tensor decomposition algorithms on four low-rank modeling computer vision benchmarks: two for image denoising, and two for background subtraction. As a baseline, we report the performance of matrix Robust PCA implemented via inexact ALM (\textit{RPCA}) \cite{Candes2011a, Lin2013}, and of Non-Negative Robust Dictionary Learning (\textit{RNNDL}) \cite{Pan:2014:RND:2892753.2892834}. We chose the following methods to include recent representatives of various existing approaches to low-rank modeling on tensors: The singleton version of Higher-Order Robust PCA (\textit{HORPCA-S}) \cite{Goldfarb2013} optimizes the Tucker rank of the tensor through the sum of the nuclear norms of its unfoldings. In \cite{Yang2015a}, the authors consider a similar model but with robust M-estimators as loss functions, either a Cauchy loss or a Welsh loss, and support both hard and soft thresholding; we tested the soft-thresholding models (\textit{Cauchy ST} and \textit{Welsh ST}). Non-convex Tensor Robust PCA (\textit{NC TRPCA}) \cite{Anandkumar2015} adapts to tensors the matrix non-convex RPCA \cite{Netrapalli2014}. Finally, the two Tensor RPCA algorithms \cite{Lu_2016_CVPR, Zhang2014} (\textit{TRPCA '14} and \textit{TRPCA '16}) work with slightly different definitions of the tensor nuclear norm as a convex surrogate of the tensor multi-rank. 


For each model, we identified a maximum of two parameters to tune via grid-search in order to keep parameter tuning tractable. When criteria or heuristics for choosing the parameters were provided by the authors, we chose the search space around the value obtained from them. In all cases, the tuning process explored a wide range of parameters to maximize performance. The range of values investigated are provided as supplementary material.

When the performance of one method was significantly worse than that of the other, the result is not reported so as not to clutter the text, and is made available in the supplementary material. This is the case of Separable Dictionary Learning \cite{Hawe2013} whose drastically different nature renders unsuitable for robust low-rank modeling, but was compared for completeness. For the same reason, we did not compare our method against K-SVD \cite{Aharon2006}, or \cite{Hsieh2014}.

\subsection{Background subtraction}

Background subtraction is a common task in computer vision and can be tackled by robust low-rank modeling: the static or mostly static background of a video sequence can effectively be represented as a low-rank tensor while the foreground forms a sparse component of outliers.

\mypar{Experimental procedure} We compared the algorithms on two benchmarks. The first is an excerpt of the \textit{Highway} dataset \cite{Goyette2012}, and consists in a video sequence of cars travelling on a highway; the background is completely static. We kept 400 gray-scale images re-sized to $48 \times 64$ pixels. The second is the \textit{Airport Hall} dataset (\cite{Li2004}) and has been chosen as a more challenging benchmark since the background is not fully static and the scene is richer. We used the same excerpt of 300 frames (frames 3301 to 3600) as in \cite{Zhao2016}, and kept the frames in their original size of $144 \times 176$ pixels.

We treat background subtraction as a binary classification problem. Since ground truth frames are available for our excerpts, we report the AUC \cite{Fawcett2006} on both videos.
The value of $\alpha$ was set to \num{1e-2} for both experiments.

\mypar{Results} We provide the original, ground truth, and recovered frames in Figure \ref{fig:visual_hall} for the \textit{Hall} experiment (\textit{Highway} in supplementary material).

Table \ref{tab:perf_bg} presents the AUC scores of the algorithms, ranked in order of their mean performance on the two benchmarks. The two matrix methods rank high on both benchmarks and only half of the tensor algorithms match or outperform this baseline. Our proposed model matches the best performance on the \textit{Highway} dataset and provides significantly higher performance than the other on the more challenging \textit{Hall} benchmark. Visual inspection of the results show KDRSDL is the only method that doesn't fully capture the immobile people in the background, and therefore achieves the best trade-off between foreground detection and background-foreground contamination.

\begin{table}[h]
\small
\centering
\resizebox{\columnwidth}{!}{
\begin{tabularx}{\columnwidth}{|>{\centering\arraybackslash}X|>{\hsize=.5\hsize\centering\arraybackslash}X|>{\hsize=.5\hsize\centering\arraybackslash}X|}\hline
\textbf{Algorithm} & \textbf{Highway} & \textbf{Hall} \\ \hline
\textbf{KDRSDL (proposed)} & 0.94 & 0.88 \\ \hline
TRPCA '16 & 0.94 & 0.86 \\ \hline
NC TRPCA   & 0.93 & 0.86 \\ \hline
\textit{RPCA (baseline)} & 0.94 & 0.85 \\ \hline
\textit{RNNDL (baseline)} & 0.94 & 0.85\\ \hline
HORPCA-S  & 0.93 & 0.86 \\ \hline
Cauchy ST & 0.83 & 0.76 \\ \hline
Welsh ST & 0.82 & 0.71 \\ \hline
TRPCA '14 & 0.76 & 0.61 \\ \hline
\end{tabularx}
}
\caption[AUC scores]{AUC on \textit{Highway} and \textit{Hall} ordered by mean AUC.}
\label{tab:perf_bg}
\normalsize
\end{table}

\begin{figure}
\captionsetup[sub]{font=scriptsize}
\begin{subfigure}[b]{.19\linewidth}
\includegraphics[width = \linewidth]{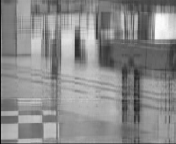} 
\caption{ Cauchy ST}
\end{subfigure}\hfill
\begin{subfigure}[b]{.19\linewidth}
\includegraphics[width = \linewidth]{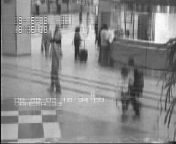} 
\caption{Welsh ST}
\end{subfigure}\hfill
\begin{subfigure}[b]{.19\linewidth}
\includegraphics[width = \linewidth]{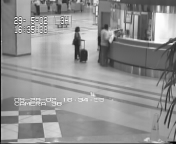}
\caption{TRPCA '16}
\end{subfigure}\hfill
\begin{subfigure}[b]{.19\linewidth}
\includegraphics[width = \linewidth]{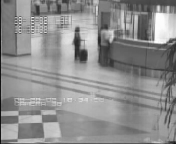} 
\caption{HORPCA-S}
\end{subfigure}\hfill
\begin{subfigure}[b]{.19\linewidth}
\includegraphics[width = \linewidth]{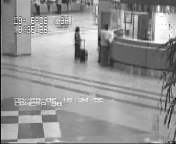} 
\caption{NCTRPCA}
\end{subfigure}\hfill
\begin{subfigure}[b]{.19\linewidth}
\includegraphics[width = \linewidth]{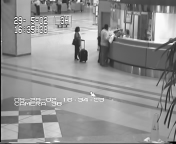} 
\caption{RNNDL}
\end{subfigure}\hfill
\begin{subfigure}[b]{.19\linewidth}
\includegraphics[width = \linewidth]{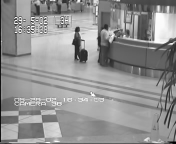} 
\caption{RPCA}
\end{subfigure}\hfill
\begin{subfigure}[b]{.19\linewidth}
\includegraphics[width = \linewidth]{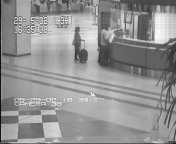}
\caption{KDRSDL}
\end{subfigure}\hfill
\begin{subfigure}[b]{.19\linewidth}
\includegraphics[width = \linewidth]{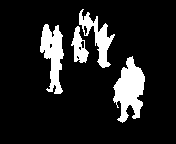}
\caption{Ground truth}
\end{subfigure}
\begin{subfigure}[b]{.19\linewidth}
\includegraphics[width = \linewidth]{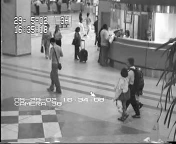} 
\caption{Original}
\end{subfigure}
\caption{Results on \textit{Airport Hall}. TRPCA '14 removed.}
\label{fig:visual_hall}
\end{figure}



\subsection{Image denoising}

Many natural and artificial images exhibit an inherent low-rank structure and are suitably denoised by low-rank modeling algorithms. In this section, we assess the performance of the cohort on two datasets chosen for their popularity, and for the typical use cases they represent.

We consider collections of grayscale images, and color images represented as 3-way tensors. Laplacian (salt \& pepper) noise was introduced separately in all frontal slices of the observation tensor at three different levels: 10\%, 30\%, and 60\% to simulate medium, high, and gross corruption. In these experiments we set the value of $\alpha$ to \num{1e-3} for noise levels up to 30\%, and to \num{1e-2} at the 60\% level.

We report two quantitative metrics designed to measure two key aspects of image recovery. The \textit{Peak Signal To Noise Ratio (PSNR)} will be used as an indicator of the element-wise reconstruction quality of the signals, while the \textit{Feature Similarity Index (FSIM, FSIMc for color images)} \cite{LinZhang2011} evaluates the recovery of structural information. Quantitative metrics are not perfect replacements for subjective visual assessment of image quality; therefore, we present sample reconstructed images for verification. Our measure of choice for determining which images to compare visually is the FSIM(c) for its higher correlation with human evaluation than the PSNR.

\mypar{Monochromatic face images} Our face denoising experiment uses the Extented Yale-B dataset \cite{Georghiades2001} of 10 different subject, each under 64 different lighting conditions. According to \cite{Ramamoorthi2001, Basri2003}, face images of one subject under various illuminations lie approximately on a 9-dimensional subspace, and are therefore suitable for low-rank modeling. We used the pre-cropped 64 images of the first subject and kept them at full resolution. The resulting collection of images constitutes a 3-way tensor of 64 images of size $192 \times 168$. Each mode corresponds respectively to the columns and rows of the image, and to the illumination component. All three are expected to be low-rank due to the spatial correlation within frontal slices and to the correlation between images of the same subject under different illuminations. We present the comparative quantitative performance of the methods tested in Figure \ref{fig:perf_yale}, and provide visualizations of the reconstructed first image at the 30\% noise level in Figure \ref{fig:visual_yale_30}. We report the metrics averaged on the 64 images.

\begin{figure}
    \centering
    \includegraphics[width=\linewidth]{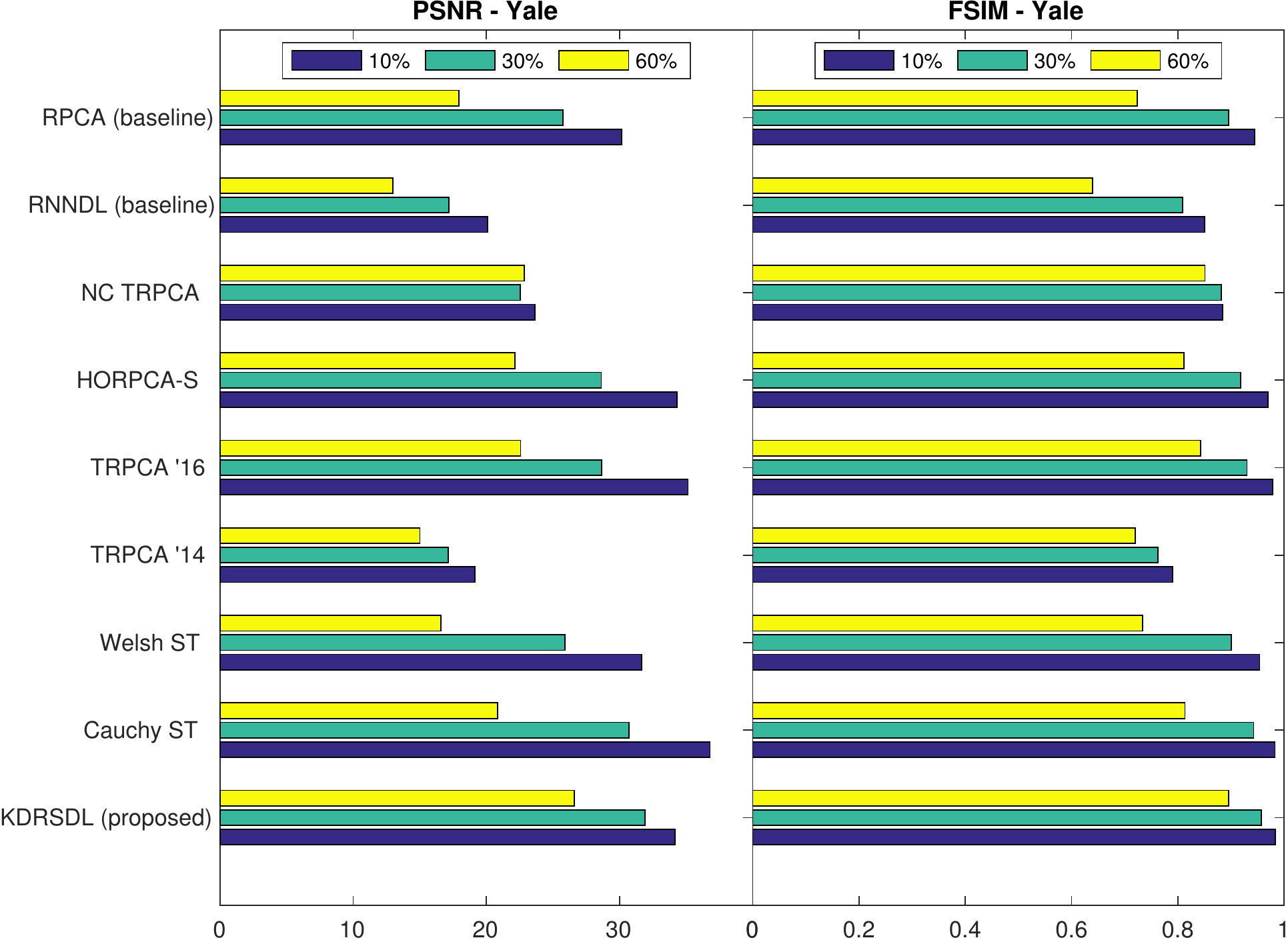}
    \caption{Mean PSNR and FSIM on the 64 images of the first subject of \textit{Yale }at noise levels 10\%, 30\%, and 60\%.}
    \label{fig:perf_yale}
\end{figure}

At the 10\% noise level, nearly every method provided good to excellent recovery of the original images. We therefore omit this noise level (\cf supplementary material). On the other hand, most methods, with the notable exception of KDRSDL, NC TRPCA, and TRPCA '16, failed to provide acceptable reconstruction in the gross corruption case. Thus, we present the denoised images at the 30\% level, and compare the performance of the three best performing methods in Table \ref{tab:top_3_yale_60} for the 60\% noise level.

Clear differences appeared at the 30\% noise level, as demonstrated both by the quantitative metrics, and by visual inspection of Figure \ref{fig:visual_yale_30}. Overall, performance was markedly lower than at the 10\% level, and most methods started to lose much of the details. Visual inspection of the results confirms a higher reconstruction quality for KDRSDL. We invite the reader to look at the texture of the skin, the white of the eye, and at the reflection of the light on the subject's skin and pupil. The latter, in particular, is very close in nature to the white pixel corruption of the salt \& pepper noise. Out of all methods, KDRSDL provided the best reconstruction quality: it is the only algorithm that removed all the noise and for which all the aforementioned details are distinguishable in the reconstruction.

\begin{figure}
\captionsetup[sub]{font=scriptsize}
\begin{subfigure}[b]{.19\linewidth}
\includegraphics[width = \linewidth]{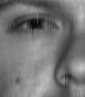} 
\caption{Cauchy ST}
\end{subfigure}\hfill
\begin{subfigure}[b]{.19\linewidth}
\includegraphics[width = \linewidth]{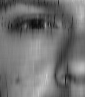} 
\caption{Welsh ST}
\end{subfigure}\hfill
\begin{subfigure}[b]{.19\linewidth}
\includegraphics[width = \linewidth]{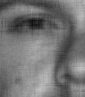}
\caption{TRPCA '16}
\end{subfigure}\hfill
\begin{subfigure}[b]{.19\linewidth}
\includegraphics[width = \linewidth]{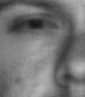} 
\caption{HORPCA-S}
\end{subfigure}\hfill
\begin{subfigure}[b]{.19\linewidth}
\includegraphics[width = \linewidth]{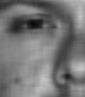} 
\caption{NCTRPCA}
\end{subfigure}\hfill
\begin{subfigure}[b]{.19\linewidth}
\includegraphics[width = \linewidth]{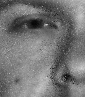} 
\caption{RNNDL}
\end{subfigure}\hfill
\begin{subfigure}[b]{.19\linewidth}
\includegraphics[width = \linewidth]{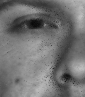} 
\caption{RPCA}
\end{subfigure}\hfill
\begin{subfigure}[b]{.19\linewidth}
\includegraphics[width = \linewidth]{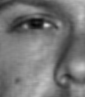}
\caption{KDRSDL}
\end{subfigure}
\begin{subfigure}[b]{.19\linewidth}
\includegraphics[width = \linewidth]{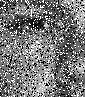}
\caption{Noisy}
\end{subfigure}
\begin{subfigure}[b]{.19\linewidth}
\includegraphics[width = \linewidth]{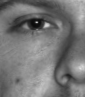} 
\caption{Original}
\end{subfigure}
\caption{Results on the Yale benchmark with 30\% noise. TRPCA '14 removed.}
\label{fig:visual_yale_30}
\end{figure}


\begin{table}[h]
    \centering
    \begin{tabular}{cccc}
        \includegraphics[width = .20 \linewidth]{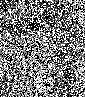} &
        \includegraphics[width = .20 \linewidth]{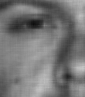} &
        \includegraphics[width = .20 \linewidth]{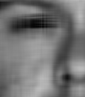} &
        \includegraphics[width = .20 \linewidth]{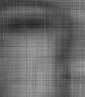}\\
         Noisy & KDRSDL & NC TRPCA & TRPCA '16 \Tstrut\Bstrut \\ \hline
         \textbf{PSNR} & 26.6057 & 22.8502 & 22.566 \Tstrut \\
         \textbf{FSIM} & 0.8956 & 0.8509 & 0.8427
    \end{tabular}
    \caption{Three best results on \textit{Yale} at 60\% noise.}
    \label{tab:top_3_yale_60}
\end{table}


At the 60\% noise level, our method scored markedly higher than its competitors on image quality metrics, as seen both in Figure \ref{fig:perf_yale} and in Table \ref{tab:top_3_yale_60}. Visualizing the reconstructions confirms the difference: the image recovered by KDRSDL at the 60\% noise level is comparable to the output of competing algorithms at the 30\% noise level.

\mypar{Color image denoising} Our benchmark is the \textit{Facade} image \cite{Chen_2016_CVPR}: the rich details and lighting makes it interesting to assess fine reconstruction. The geometric nature of the building's front wall, and the strong correlation between the RGB bands indicate the data can be modeled by a low-rank 3-way tensor where each frontal slice is a color channel.

\begin{figure}
    \centering
    \includegraphics[width=\linewidth]{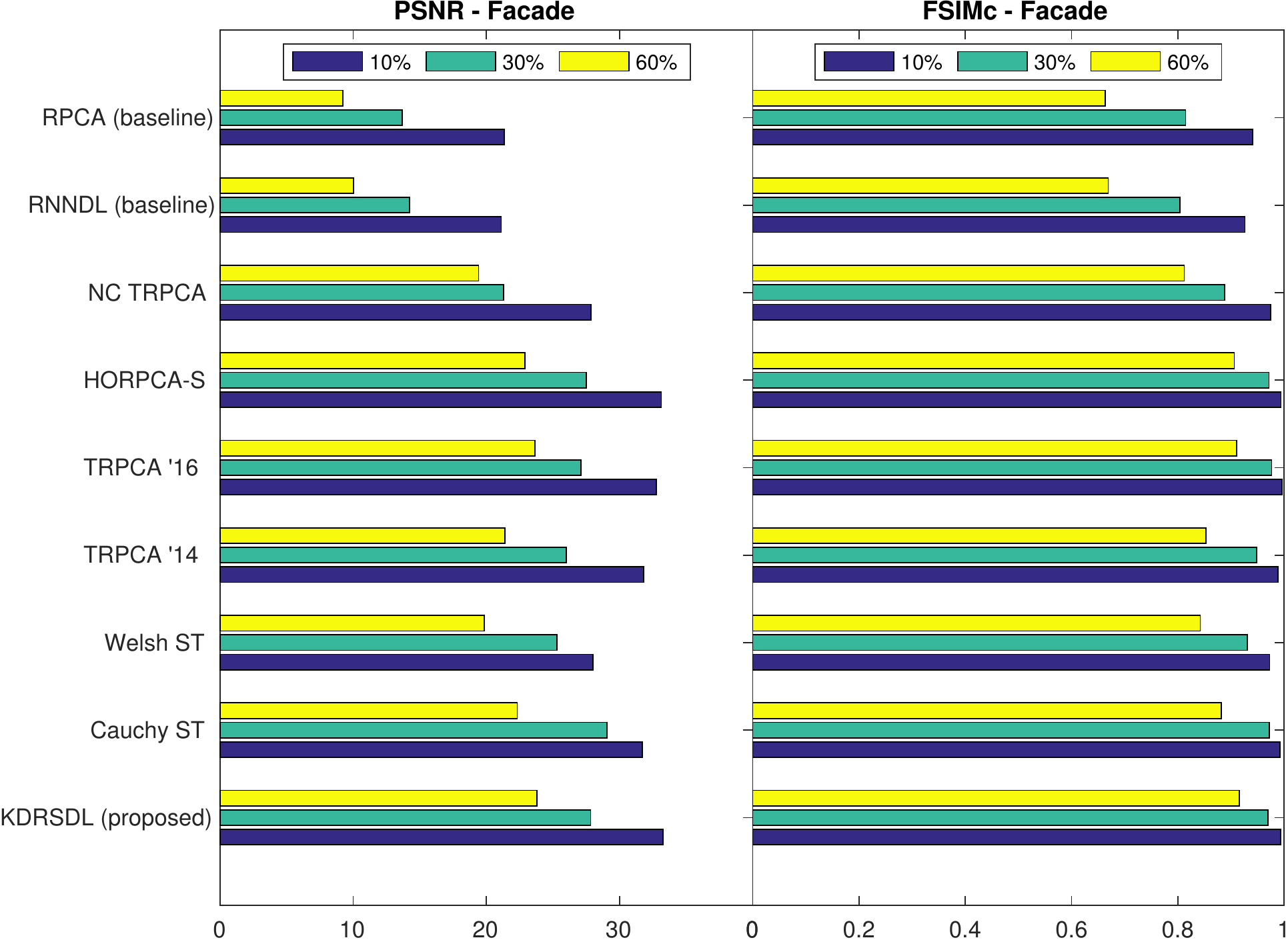}
    \caption{PSNR and FSIMc of all methods on the Facade benchmark at noise levels 10\%, 30\%, and 60\%.}
    \label{fig:perf_facade}
\end{figure}

At the 10\% noise level, KDRSDL attained the highest PSNR, and among the highest FSIMc values. Most methods provided excellent reconstruction, in agreement with the high values of the metrics shown in Figure \ref{fig:perf_facade}. As in the previous benchmark, the results are omitted because of the space constraints (\cf supplementary material). At the 30\% noise level, Cauchy ST exhibited the highest PSNR, while TRPCA '16 scored best on the FSIMc metric. KDRSDL had the second highest PSNR and among the highest FSIMc. Details of the results are provided in Figure \ref{fig:visual_facade_30}. Visually, clear differences are visible, and are best seen on the fine details of the picture, such as the black iron ornaments, or the light coming through the window. Our method best preserved the dynamics of the lighting, and the sharpness of the details, and in the end provided the reconstruction visually closest to original. Competing models tend to oversmooth the image, and to make the light dimmer; indicating substantial losses of high-frequency and dynamic information. KDRSDL appears to also provide the best color fidelity.
\begin{figure}
\captionsetup[sub]{font=scriptsize}
\begin{subfigure}[b]{.19\linewidth}
\includegraphics[width = \linewidth]{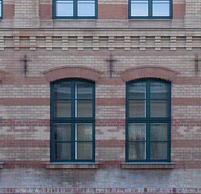} 
\caption{Cauchy ST}
\end{subfigure}\hfill
\begin{subfigure}[b]{.19\linewidth}
\includegraphics[width = \linewidth]{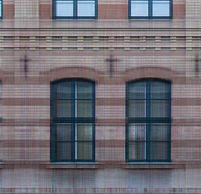} 
\caption{Welsh ST}
\end{subfigure}\hfill
\begin{subfigure}[b]{.19\linewidth}
\includegraphics[width = \linewidth]{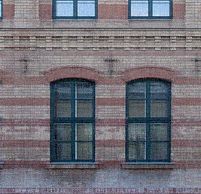} 
\caption{TRPCA '14}
\end{subfigure}\hfill
\begin{subfigure}[b]{.19\linewidth}
\includegraphics[width = \linewidth]{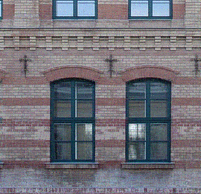}
\caption{TRPCA '16}
\end{subfigure}\hfill
\begin{subfigure}[b]{.19\linewidth}
\includegraphics[width = \linewidth]{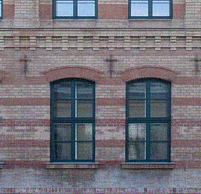} 
\caption{HORPCA-S}
\end{subfigure}\hfill
\begin{subfigure}[b]{.19\linewidth}
\includegraphics[width = \linewidth]{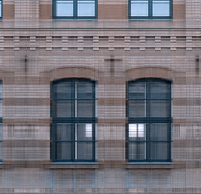} 
\caption{NCTRPCA}
\end{subfigure}\hfill
\begin{subfigure}[b]{.19\linewidth}
\includegraphics[width = \linewidth]{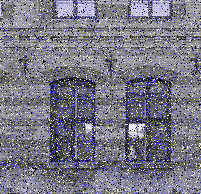} 
\caption{RNNDL}
\end{subfigure}\hfill
\begin{subfigure}[b]{.19\linewidth}
\includegraphics[width = \linewidth]{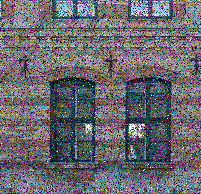} 
\caption{RPCA}
\end{subfigure}\hfill
\begin{subfigure}[b]{.19\linewidth}
\includegraphics[width = \linewidth]{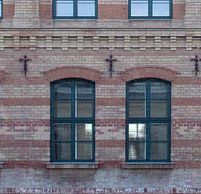}
\caption{KDRSDL}
\end{subfigure}
\begin{subfigure}[b]{.19\linewidth}
\includegraphics[width = \linewidth]{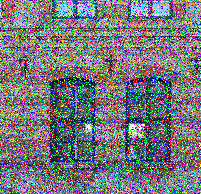}
\caption{Noisy}
\end{subfigure}
\begin{subfigure}[b]{.19\linewidth}
\includegraphics[width = \linewidth]{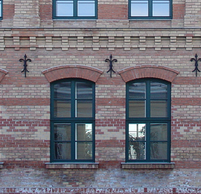} 
\caption{Original}
\end{subfigure}
\caption{Results on the Facade benchmark with 30\% noise.}
\label{fig:visual_facade_30}
\end{figure}



\begin{table}[h]
    \centering
    \begin{tabular}{cccc}
        \includegraphics[width = .20 \linewidth]{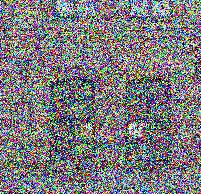} &
        \includegraphics[width = .20 \linewidth]{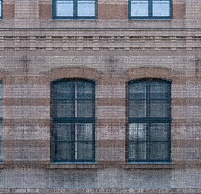} &
        \includegraphics[width = .20 \linewidth]{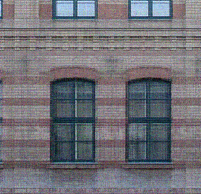} &
        \includegraphics[width = .20 \linewidth]{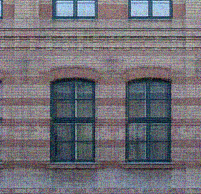}\\
         Noisy & KDRSDL & TRPCA '16 & HORPCA-S \Tstrut\Bstrut \\ \hline
         \textbf{PSNR} & 23.8064 & 23.6552 & 22.8811 \Tstrut \\
         \textbf{FSIMc} & 0.9152 & 0.9109 & 0.9060
    \end{tabular}
    \caption{Three best results on \textit{Facade} at 60\% noise.}
    \label{tab:top_3_facade_60}
\end{table}


In the gross-corruption case, KDRSDL was the best performer on both PSNR and FSIMc. As seen on Figure \ref{tab:top_3_facade_60}, KDRSDL was the only method with TRPCA '16 and HORPCA-S to provide a reconstruction with distinguishable details, and did it best.

\section{Conclusion}

The method we propose combines aspects from Robust Principal Component Analysis, and Sparse Dictionary Learning. As in K-SVD, our algorithm learns iteratively and alternatively both the dictionary and the sparse representations. Similarly to Robust PCA, our method assumes the data decompose additively in a sparse and low-rank component, and is able to separate the two signals. We introduce a two-level structure in the dictionary to allow for both scalable training, and robustness to outliers. Imposing this structure exhibits links with tensor factorizations and allows us to better model spatial correlation in images than classical matrix methods. These theoretical advantages translate directly in the experimental performance, as our method exhibits desirable scalability properties, and matches or outperforms the current state of the art in low-rank modeling.

\section{Acknowledgements}
The work of Y. Panagakis has been partially supported by the European Community Horizon 2020 [H2020/2014-2020] under Grant Agreement No. 645094 (SEWA). S. Zafeiriou was partially funded by EPSRC Project EP/N007743/1 (FACER2VM).

{\small
\bibliographystyle{ieee}
\bibliography{ms,fix}
}
\end{document}